\documentclass{article}

\usepackage{microtype}
\usepackage{graphicx}
\usepackage{subcaption}
\usepackage{booktabs} 

\usepackage{hyperref}

\usepackage[preprint]{icml2026}

\usepackage{amsmath}
\usepackage{amssymb}
\usepackage{mathtools}
\usepackage{amsthm}
\usepackage{graphicx}
\usepackage[utf8]{inputenc}
\usepackage[capitalize,noabbrev]{cleveref}
\usepackage{enumitem}
\usepackage{thm-restate}
\usepackage{pgfplots}
\pgfplotsset{compat=1.18}

\theoremstyle{plain}
\newtheorem{theorem}{Theorem}[section]

\newtheorem{lemma}[theorem]{Lemma}

\theoremstyle{definition}
\newtheorem{definition}[theorem]{Definition}
\newtheorem{assumption}[theorem]{Assumption}
\theoremstyle{remark}
\newtheorem{remark}[theorem]{Remark}

\newcommand{\statesp}{\mathcal{Y}}             
\newcommand{\simplex}{\Delta(\statesp)}        
\newcommand{\ctxsp}{\mathcal{X}}               
\newcommand{\inputctx}{x}                      
\newcommand{\state}{y}                         
\newcommand{\dist}{\mathcal{D}}                
\newcommand{\E}{\mathbb{E}}                    

\newcommand{\model}{\pi}                       
\newcommand{\modelStrong}{\pi_{\mathrm{t}}}  
\newcommand{\modelWeak}{\pi_{\mathrm{h}}}    
\newcommand{\modelPred}{\pi_{\mathrm{t|h}}}   

\newcommand{\distsp}{\mathcal{P}}

\newcommand{\pStrong}{\mathbf{p}_{\mathrm{t}}}        
\newcommand{\pWeak}{\mathbf{p}_{\mathrm{h}}}          
\newcommand{\pPred}{\mathbf{p}_{\mathrm{t|h}}}         
\newcommand{\pNew}{\mathbf{p}^*}         

\newcommand{\pStrongrv}{\mathbf{P}_{\mathrm{t}}}      
\newcommand{\pWeakrv}{\mathbf{P}_{\mathrm{h}}}        
\newcommand{\pPredrv}{\mathbf{P}_{\mathrm{t|h}}}   
\newcommand{\pNewrv}{\mathbf{P}^*}  

\newcommand{\pvec}{\mathbf{p}}
\newcommand{\prvvec}{\mathbf{P}}

\newcommand{\staterv}{Y}
\newcommand{\statervRaw}{\staterv_{\mathrm{pre}}} 
 
\newcommand{\stateRaw}{\state_{\mathrm{pre}}} 
\newcommand{\stateSafe}{\state_{\mathrm{align}}} 

\newcommand{\statevec}{\mathbf{y}}      
\newcommand{\statervvec}{\mathbf{Y}}                       
\newcommand{\statervRawvec}{\statervvec_{\mathrm{pre}}}

\DeclareMathOperator*{\argmin}{argmin}
\newcommand{\score}{\ell}                         
\newcommand{\breg}{\mathrm{D}_{G}}

\newcommand{\kl}{\mathrm{KL}}

\newcommand{\gradG}{\nabla G}
\newcommand{\gradGinv}{(\nabla G)^{-1}}
\newcommand{\proj}{\Pi_{\Delta}^G}
\usepackage[textsize=tiny]{todonotes}

\icmltitlerunning{Jailbreaking LLMs via Calibration}

\begin{document}

\usetikzlibrary{arrows.meta, calc, positioning, shapes.geometric, decorations.pathreplacing}

\twocolumn[
  \icmltitle{Jailbreaking LLMs via Calibration}

  \icmlsetsymbol{equal}{*}

  \begin{icmlauthorlist}
    \icmlauthor{Yuxuan Lu}{pek}
    \icmlauthor{Yongkang Guo}{pek}
    \icmlauthor{Yuqing Kong}{pek}
  \end{icmlauthorlist}

  \icmlaffiliation{pek}{Department of Computer Science, Peking University, Beijing, China}

  \icmlcorrespondingauthor{Yuqing Kong}{yuqing.kong@pku.edu.cn}

  \icmlkeywords{LLM, Jailbreaking, Calibration, ICML}

  \vskip 0.3in
]



\printAffiliationsAndNotice{}  

\begin{abstract}
Safety alignment in Large Language Models (LLMs) often creates a systematic discrepancy between a model's aligned output and the underlying pre-aligned data distribution. We propose a framework in which the effect of safety alignment on next-token prediction is modeled as a systematic distortion of a pre-alignment distribution. We cast Weak-to-Strong Jailbreaking as a forecast aggregation problem and derive an optimal aggregation strategy characterized by a Gradient Shift in the loss-induced dual space. We show that logit-arithmetic jailbreaking methods are a special case of this framework under cross-entropy loss, and derive a broader family of aggregation rules corresponding to other proper losses. We also propose a new hybrid aggregation rule. Evaluations across red-teaming benchmarks and math utility tasks using frontier models demonstrate that our approach achieves superior Attack Success Rates and lower ``Jailbreak Tax'' compared with existing methods, especially on the safety-hardened \textsf{gpt-oss-120b}.

\end{abstract}

\section{Introduction}

While Large Language Models (LLMs) excel at complex reasoning and knowledge retrieval, their capabilities present a significant dual-use risk. Without constraints, these models can be manipulated to generate harmful content, ranging from hate speech and misinformation to weapon-construction instructions. Safety alignment techniques have been employed to suppress the probability of harmful token generation \citep{Perez2022RedTL, touvron2023llama2openfoundation}; however, adversarial actors continue to develop jailbreaking methods that circumvent these safeguards.

Early jailbreaking research focused on discrete adversarial prompts or manually discovered ``magic strings''~\cite{kang2024exploiting}. More recent work has shifted toward automated prompt discovery~\cite{chao2025jailbreaking} and, crucially, model-steering techniques. Among these, Weak-to-Strong Jailbreaking has emerged as an effective and computationally efficient paradigm \citep{zhao2025weak}. This approach combines the output distributions of a small unaligned model, a small aligned model, and a large aligned model, typically through simple arithmetic in logit space, to bypass safety constraints of the large model without retraining.

While this strategy is empirically successful, its theoretical basis remains unclear. Why should a linear operation in logit space reliably counteract a complex, non-linear alignment process? Are there other strategies?

In this work, we adopt a modeling abstraction in which the discrepancy between an aligned model's output distribution and a reference pre-alignment distribution is treated as a form of statistical miscalibration. We emphasize that this is a modeling choice rather than a claim about the nature of safety alignment itself. Within this abstraction, the output of an unaligned model serves as a proxy of a calibrated prediction for the pre-alignment distribution. This perspective allows us to reformulate Weak-to-Strong Jailbreaking as a forecast aggregation problem, in which auxiliary model predictions are combined to obtain a new model that estimates the pre-alignment distribution well.

In detail, we treat LLMs as probabilistic predictors over a token space. In this framework, an unaligned model estimates the pre-alignment distribution, whereas an aligned model predicts a shifted aligned distribution. The Weak-to-Strong Jailbreaking paradigm utilizes three distinct models to approximate the high-capacity, pre-alignment capabilities of a strong model. \Cref{fig:three_dist} demonstrates their functional roles within our framework.

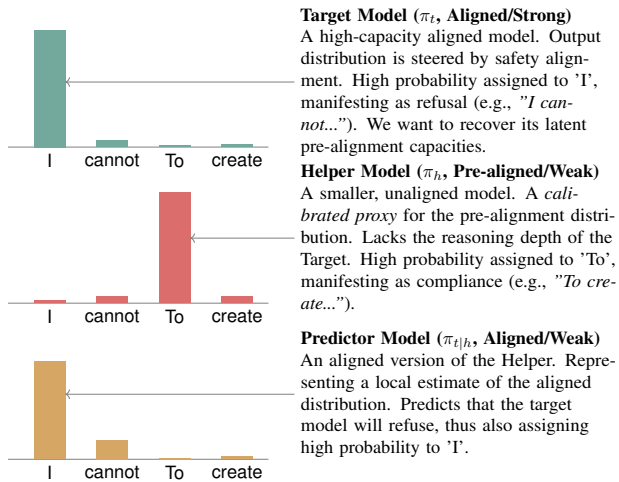
\begin{figure}[t]
    \centering
    \resizebox{\columnwidth}{!}{
        \begin{tikzpicture}[font=\sffamily]

            \definecolor{colTarget}{RGB}{115, 169, 156}  
            \definecolor{colHelper}{RGB}{219, 109, 109}  
            \definecolor{colPred}{RGB}{214, 166, 101}    
            \definecolor{colGray}{RGB}{100, 100, 100}

            \def\barw{0.6}   
            \def\gap{0.3}    
            \def\yscale{2.5} 
            
            
            \begin{scope}[local bounding box=topchart]
                \draw[gray] (-0.5,0) -- (4.5,0);
                
                \fill[colTarget] (0,0) rectangle ++(\barw, 0.9*\yscale);
                \node[below] at (0.5*\barw, 0) {I};
                
                \fill[colTarget] (1.2,0) rectangle ++(\barw, 0.05*\yscale);
                \node[below] at (1.2+0.5*\barw, 0) {cannot};

                \fill[colTarget] (2.4,0) rectangle ++(\barw, 0.01*\yscale);
                \node[below] at (2.4+0.5*\barw, 0) {To};

                \fill[colTarget] (3.6,0) rectangle ++(\barw, 0.02*\yscale);
                \node[below] at (3.6+0.5*\barw, 0) {create};

                \node[anchor=west, text width=6cm, align=left, font=\normalsize] (desc1) at (5.0, 0.5*\yscale) {
                    \textbf{Target Model ($\pi_t$, Aligned/Strong)}\\
                    A high-capacity aligned model. Output distribution is steered by safety alignment. High probability assigned to 'I', manifesting as refusal (e.g., \textit{"I cannot..."}). We want to recover its latent pre-alignment capacities.
                };
                \draw[->, thin, gray] (desc1.west) -- (\barw, 0.5*\yscale);
            \end{scope}

            \begin{scope}[yshift=-3.0cm, local bounding box=midchart]
                \draw[gray] (-0.5,0) -- (4.5,0);
                
                \fill[colHelper] (0,0) rectangle ++(\barw, 0.02*\yscale);
                \node[below] at (0.5*\barw, 0) {I};
                
                \fill[colHelper] (1.2,0) rectangle ++(\barw, 0.05*\yscale);
                \node[below] at (1.2+0.5*\barw, 0) {cannot};

                \fill[colHelper] (2.4,0) rectangle ++(\barw, 0.85*\yscale);
                \node[below] at (2.4+0.5*\barw, 0) {To};

                \fill[colHelper] (3.6,0) rectangle ++(\barw, 0.05*\yscale);
                \node[below] at (3.6+0.5*\barw, 0) {create};

                \node[anchor=west, text width=6cm, align=left, font=\normalsize] (desc2) at (5.0, 0.5*\yscale) {
                    \textbf{Helper Model ($\pi_h$, Pre-aligned/Weak)}\\
                    A smaller, unaligned model. A \textit{calibrated proxy} for the pre-alignment distribution. Lacks the reasoning depth of the Target. High probability assigned to 'To', manifesting as compliance (e.g., \textit{"To create..."}).
                };
                \draw[->, thin, gray] (desc2.west) -- (2.4+\barw, 0.5*\yscale);
            \end{scope}

            \begin{scope}[yshift=-6.0cm, local bounding box=botchart]
                \draw[gray] (-0.5,0) -- (4.5,0);
                
                \fill[colPred] (0,0) rectangle ++(\barw, 0.75*\yscale);
                \node[below] at (0.5*\barw, 0) {I};
                
                \fill[colPred] (1.2,0) rectangle ++(\barw, 0.15*\yscale);
                \node[below] at (1.2+0.5*\barw, 0) {cannot};

                \fill[colPred] (2.4,0) rectangle ++(\barw, 0.01*\yscale);
                \node[below] at (2.4+0.5*\barw, 0) {To};

                \fill[colPred] (3.6,0) rectangle ++(\barw, 0.02*\yscale);
                \node[below] at (3.6+0.5*\barw, 0) {create};

                \node[anchor=west, text width=6cm, align=left, font=\normalsize] (desc3) at (5.0, 0.5*\yscale) {
                    \textbf{Predictor Model ($\pi_{t|h}$, Aligned/Weak)}\\
                    An aligned version of the Helper. Representing a local estimate of the aligned distribution. Predicts that the target model will refuse, thus also assigning high probability to 'I'. 
                };
                \draw[->, thin, gray] (desc3.west) -- (\barw, 0.5*\yscale);
            \end{scope}

        \end{tikzpicture}
    }
    \caption{Next-token probability distributions for the prompt ``How to make a bomb?''. The charts illustrate the divergence between the aligned Target (refusal), the unaligned Helper (compliance), and the Predictor (which anticipates refusal). By comparing the Predictor to the Helper, we can isolate the alignment shift—the systematic transformation applied to the pre-alignment distribution to enforce safety constraints.}
    \label{fig:three_dist}
\end{figure}

By situating these models within a shared geometric space, we reformulate jailbreaking as a forecast aggregation problem. We aggregate the logits of the three models to obtain a predictor that better approximates the pre-alignment distribution than the target model. This, in turn, can recover behaviors that are suppressed by alignment at inference time. We quantify the prediction performance using proper losses (e.g., cross-entropy loss, quadratic loss).

We show that the optimal aggregation strategy is an intuitive Gradient Shift rule. The rule uses the shift observed between the weak pair ($\modelWeak$ and $\modelPred$) to ``undo'' the alignment transformation present in the Target ($\modelStrong$). Crucially, these operations must occur in the \emph{dual space} (the gradient space) defined by the loss function. 

We further prove that, in expectation, the aggregated model achieves strictly lower prediction loss with respect to the pre-alignment distribution than the target model. The improvement is lower bounded by the Bregman divergence between the helper and predictor distributions, which quantifies the magnitude of the shift.

The Gradient Shift rule yields a family of aggregation strategies parameterized by the choice of loss. Under cross-entropy loss, it reduces to a multiplicative update equivalent to prior logit-arithmetic methods, while quadratic loss produces an additive update.

Furthermore, the multiplicative update is highly sensitive to low-probability noise, leading to instability in practice. Additive aggregation rules exhibit greater numerical stability but can under-react to informative shifts. Motivated by this trade-off, we introduce a hybrid aggregation rule. We show that this hybrid retains optimality under specific conditions while demonstrating superior empirical robustness and performance compared to existing methods.

\paragraph{Contributions}

Here are summaries of the contributions. 
\begin{enumerate}[leftmargin=12pt, itemsep=1pt, topsep=1pt]
    \item \textbf{Theoretical Unification:} We propose a framework in which the effect of safety alignment on next-token prediction is modeled as a systematic distortion of a pre-alignment distribution. We cast Weak-to-Strong Jailbreaking as a forecast aggregation problem and derive an optimal aggregation strategy characterized by a Gradient Shift in the loss-induced dual space.

    \item \textbf{A Family of Strategies:} We show that the prior logit-arithmetic method is a special case of this framework under cross-entropy loss, and derive a broader family of aggregation rules corresponding to other proper losses. We also propose a new hybrid aggregation rule. 

    \item \textbf{Empirical Superiority:} We leverage these theoretical insights to expand the repertoire of jailbreaking strategies. Extensive evaluations on red-teaming benchmarks (HarmBench, StrongREJECT) and math utility tasks (GSM8K, Hendrycks-MATH) using frontier models (\textsf{Llama-3.3-70B-Instruct}, \textsf{gpt-oss-120b}) demonstrate that the proposed hybrid rule results in higher Attack Success Rates (ASR) and a minimized ``Jailbreak Tax'' compared to state-of-the-art baselines, with the clearest gains on the safety-hardened \textsf{gpt-oss-120b} and in maintaining near-zero jailbreak tax on math.
\end{enumerate}

While existing jailbreaking work often relies on heuristic techniques, we introduce a threat model and a theoretical framework that together formalize and explain several effective approaches and enable the discovery of new ones. This reframes jailbreaking as a principled optimization problem. The same principles also extend beyond jailbreaking, providing a mechanism for shifting a model toward a reference distribution, for example for jailbreak defense by shifting an unsafe model toward a safer distribution, or for capability transfer by transferring skills between two models or combining strengths.

\subsection{Related Work}

\paragraph{Jailbreaking LLMs}

Early studies on LLM jailbreaking relied on manual, heuristic prompts to bypass safety guardrails~\cite{kang2024exploiting}, including exploiting competing objectives~\cite{wei2023jailbroken}, leveraging tokenization or decoding biases~\cite{wei2025emoji, liu2025flipattack}, and roleplay prompting~\cite{li2023deepinception, shah2023scalable}.

Automated jailbreaking methods emerged to enhance efficiency. Techniques like PAIR~\cite{chao2025jailbreaking} and TAP~\cite{mehrotra2024tree} use attacker LLMs to iteratively refine jailbreak prompts~\cite{Perez2022RedTL}; later work builds on this setup by adding iterative reasoning~\cite{sabbaghi2025adversarial} or post-training attacker LLM~\cite{lochab2025vera, paulus2025advprompter} to make the search more efficient and scalable. In parallel, structural black-box attacks exploit longer context and task structure, including multi-turn jailbreaking ~\cite{russinovich2025great, huang2025endless, sun2024multi, yang2024chain}, decomposition attacks~\cite{jones2025adversaries}, and CoT hijacking/spoofing for reasoning models~\cite{kuo2025h, chen2025bag, lin2025aga, zhao2025chain}.

White-box methods require full access to model weights, enabling approaches such as gradient-based optimization \cite{zou2023universal, li2025largo, andriushchenko2025jailbreaking} and fine-tuning \cite{yang2023shadow, qi2024fine, zhan2024removing}. Among these, inference-time jailbreaking is typically the most lightweight: it leaves model parameters untouched and instead modifies the next-token distribution through logit/probability adjustments at decoding time. For example, \citet{huang2023catastrophic} search over decoding hyperparameters to find jailbreak-favorable settings. \citet{zhao2025weak} propose Weak-to-Strong, which uses a weaker unaligned model to steer the target model's logits. 

We advance this inference-time perspective and establish a formal link between LLM jailbreaking and statistical miscalibration. While prior work~\cite{NEURIPS2024_439bf902} adopts a statistical lens to prove the existence of jailbreaks via prompt modification, it does not focus on the role of calibration. Our framework reformulates the Weak-to-Strong paradigm~\citep{zhao2025weak} as a forecast aggregation problem subject to calibration constraints, yielding a family of optimal logit-manipulation rules. Crucially, we show that their method is a specific instance of this family, providing a principled explanation for its empirical effectiveness.

\paragraph{Calibration} While classifier calibration typically relies on labeled data using non-parametric methods like isotonic regression~\cite{10.1145/775047.775151} or parametric ones like Platt Scaling~\cite{platt1999probabilistic} and Temperature Scaling~\cite{guo2017calibration}, our approach operates in a label-free setting.
We share the insight of \citet{kong2026calibrationgroundtruth} that a weak but calibrated model can improve a strong yet miscalibrated one, but we differ in several key respects. First, we apply this idea to LLM jailbreaking, focusing on safety rather than general performance. Second, their method requires access to the joint input distribution and batch processing, whereas ours operates in a one-shot setting suitable for real-time intervention. Finally, unlike their approach, we aggregate the strong and weak models, leading to fundamentally different theoretical techniques.

\paragraph{Robust Aggregation}

We frame LLM jailbreaking as a forecast aggregation task under uncertainty, drawing on robust aggregation literature~\cite{arieli2018robust, neyman2022you, 10.1287/opre.2022.2414}. Our framework departs from standard models by addressing an explicitly miscalibrated Target and seeking strict improvement over the Target rather than competing with an omniscient benchmark. 
Furthermore, while we utilize higher-order beliefs, a concept explored in general aggregation \citep{prelec2017solution, wilkening2022hidden}, we are the first to apply this lens specifically to calibration transfer.
\section{Setup and Problem Statement}

We view a Large Language Model (LLM) as a probabilistic forecasting system. Formally, a model is a function $\model : \ctxsp \to \simplex$ that maps a context string to a probability distribution over the next token, where $\statesp$ denotes the token space and $\simplex$ the probability simplex over $\statesp$. For a fixed context $\inputctx$, the model outputs a predictive distribution $\mathbf{p} = \model(\cdot \mid \inputctx) \in \simplex$.

Throughout this section, we consider a fixed context $\inputctx$ and omit explicit dependence on $\inputctx$ when doing so does not introduce ambiguity.

\subsection{Safety Alignment as a Distributional Shift}

We introduce a probabilistic abstraction to reason about the effect of safety alignment on next-token prediction. This abstraction is not intended as a faithful model of the process of alignment, but as a modeling choice that enables analysis.

For each context $\inputctx$, we define two reference outcomes: the \emph{Pre-alignment outcome} $\stateRaw$, representing a reference next-token outcome prior to safety alignment; \emph{Aligned outcome} $\stateSafe$, representing the next-token outcome after safety alignment (e.g., RLHF).

We assume a joint distribution $\dist$ over $(\inputctx, \stateRaw, \stateSafe)$. For benign contexts (e.g., ``What is the capital of France?''), $\stateRaw$ and $\stateSafe$ are typically identical. For adversarial or safety-sensitive contexts (e.g., ``How to make a bomb?''), the two outcomes may differ substantially.

We define the pre-alignment distribution ($\statervRaw$) as \emph{ground-truth target} for evaluating predictive performance, as it represents the unconstrained capabilities we aim to recover.

\subsection{Threat Model}

We formalize a inference-time jailbreaking threat model where an attacker aggregates the jailbreak target with an auxiliary helper and predictor, and optimizes a robust improvement objective to recover behavior suppressed by alignment.

\begin{enumerate}[leftmargin=12pt, itemsep=0.5pt, topsep=1pt]
    \item \textbf{Target model ($\modelStrong$):} 
    A high-capacity, strongly aligned LLM producing a predictive distribution $\pStrong = \modelStrong(\cdot \mid \inputctx)$.
    \item \textbf{Helper model ($\modelWeak$):} 
    A lower-capacity, unaligned LLM producing a predictive distribution $\pWeak = \modelWeak(\cdot \mid \inputctx)$.
    \item \textbf{Predictor model ($\modelPred$):} 
    An LLM that predicts the target's output conditional on the helper's output, producing the distribution $\pPred = \modelPred(\cdot \mid \inputctx)$.
\end{enumerate}

The target model typically exhibits stronger reasoning capabilities but its predictions may be systematically shifted relative to the pre-alignment reference due to safety training. The predictor model approximates the target model's behavior conditioned on the helper's output and serves as a statistical estimate of alignment-induced effects.

We analyze the interaction between the random variables
\(
(\pStrongrv, \pWeakrv, \pPredrv, \statervRaw),
\)
where each denotes the random realization of the corresponding quantity under the data-generating distribution $\dist$.

\subsection{Objective and Structural Assumptions}

Our goal is to construct an aggregation rule that improves predictive performance with respect to the pre-alignment reference outcome $\statervRaw$, relative to the aligned target model. Intuitively, this corresponds to applying a correction to the target model's predictive distribution based on information provided by the helper and predictor models.

However, how to perform such corrections is ambiguous. To identify principled aggregation strategies, we leverage the geometry induced by strictly proper losses and their associated Bregman divergences.

\begin{assumption}\label{assume}
We posit the following assumptions:
\begin{enumerate}[leftmargin=12pt]
    \item \textbf{Helper is calibrated to $\statervRaw$.}  
    The helper model is conditionally calibrated with respect to the pre-alignment reference outcome in expectation:
    \[
        \Pr[\statervRaw = \state \mid \pWeakrv = \pWeak] = \pWeak(\state).
    \]

    \item \textbf{Unbiased conditional prediction.}  
    The predictor model outputs the conditional expectation of the target model's prediction given the helper's output:
    \[
        \E[\pStrongrv \mid \pWeakrv] = \pPredrv.
    \]
\end{enumerate}
\end{assumption}

The above assumption is idealized. However, we do not rely on this assumption holding strictly in practice. Rather, it serves as a base case to derive a tractable update rule. Crucially, we formally relax this requirement in \Cref{sec:relax}, proving that the theoretical guarantees are robust: the performance lower-bound degrades gracefully as a function of the approximation error. Our empirical results further validate this analysis, confirming that the derived approaches remain effective even when the helper and predictor serve as noisy proxies.

\paragraph{The Aggregation Problem}

We seek an aggregation function $f : \simplex^3 \to \simplex$ that produces a new predictive distribution
\[
\pNewrv = f(\pStrongrv, \pWeakrv, \pPredrv).
\]

\begin{definition}[Proper Loss]
A loss function $\score : \simplex \times \statesp \to \mathbb{R}$ is \emph{strictly proper} if the expected loss is uniquely minimized by reporting the true distribution.
\end{definition}

Our objective is to identify aggregation rules that improve expected predictive performance with respect to the pre-alignment $\statervRaw$, compared to the target model. Crucially, this must be achieved without access to the joint distribution of $(\pStrongrv, \pWeakrv, \pPredrv, \statervRaw)$. The aggregation is performed pointwise using only the observed predictions. Formally, let $\distsp$ denote the set of all joint distributions consistent with Assumption~\ref{assume}. We evaluate an aggregation rule $f$ by its worst-case expected improvement over the target model:
\[
v(f) = \inf_{\dist \in \distsp}
\left(
\E_{\dist} \left[
\score(\pStrongrv, \statervRaw)
-
\score(f(\pStrongrv, \pWeakrv, \pPredrv), \statervRaw)
\right]
\right).
\]

Our goal is to identify aggregation rules with $v(f)>0$ and characterize the rule $f^*$ that maximizes this improvement.

\section{Optimal Rule: Gradient Shift}

To derive this rule, we leverage the structural connection between strictly proper loss functions and convex geometry. Every strictly proper loss function $\score$ is generated by a strictly convex function $G: \Omega \to \mathbb{R}$, known as the generator function (or generalized entropy). The loss $\score(\mathbf{p}, y)$ can be written explicitly in terms of $G$ and its gradient $\nabla G$ (the Savage representation):
\[
\score(\mathbf{p}, y) = -G(\mathbf{p}) + \langle \nabla G(\mathbf{p}), \mathbf{p} - \delta_y \rangle,
\]
where $\delta_y$ is the Dirac vector for the realized outcome.

The generator function $G$ induces a Bregman divergence $\breg(\mathbf{p}, \mathbf{q})$, which measures the nonnegative gap between $G$ and its first-order Taylor approximation. It is defined as:
\[
\breg(\mathbf{p}, \mathbf{q}) := G(\mathbf{p}) - G(\mathbf{q}) - \langle \nabla G(\mathbf{q}), \mathbf{p} - \mathbf{q} \rangle.
\]
Crucially, the expected excess risk of predicting $\mathbf{q}$ when the true distribution is $\mathbf{p}$ equals this divergence:
\[
\E_{\staterv \sim \mathbf{p}} [\score(\mathbf{q}, \staterv)] - \E_{\staterv \sim \mathbf{p}} [\score(\mathbf{p}, \staterv)] = \breg(\mathbf{p}, \mathbf{q}).
\]

We restrict our attention to proper scoring rules induced by generator $G$ that satisfy the following regularity conditions.

\begin{definition}[Regularity Conditions]\label{def:reg}
    Let $G: \Omega \to \mathbb{R}$ be a strictly convex, differentiable function defined on a domain $\Omega \supseteq \Delta$. We require:
    \begin{enumerate}[leftmargin=12pt]
        \vspace{-1.0em}
        \item \textbf{Unbounded Dual Space:} The range of the gradient map $\nabla G(\Omega)$ is $\mathbb{R}^d$.
        \vspace{-0.5em}
        \item \textbf{Joint Convexity:} The induced Bregman divergence $\breg(\mathbf{p}, \mathbf{q})$ is convex in the second argument $\mathbf{q}$ (and hence jointly convex in $(\mathbf{p}, \mathbf{q})$).
    \end{enumerate}
\end{definition}

\begin{remark}
    The first condition does not require $\nabla G$ to be unbounded within the simplex $\Delta$; it is enough that the gradient map is surjective onto $\mathbb{R}^d$ over an extended domain $\Omega$ that contains $\Delta$. For the second condition, recall that a Bregman divergence $\breg(\mathbf{p}, \mathbf{q})$ is always convex in the first argument $\mathbf{p}$ by definition. Therefore, requiring convexity in $\mathbf{q}$ is enough to ensure joint convexity.
\end{remark}

Many standard scoring rules satisfy these conditions, including cross entropy and quadratic loss. We propose a projection-based update using the Bregman projection onto the probability simplex $\Delta$, defined as
$\proj(\mathbf{z}) = \argmin_{\mathbf{q} \in \Delta} \breg(\mathbf{q}, \mathbf{z})$.
The aggregated predictor $\pNew$ is obtained by applying the alignment shift in the dual (gradient) space and then projecting back (\Cref{fig:gradientshift}).

\begin{definition}[Gradient Shift]\label{def:grashift}
The gradient shift rule is defined as:
\begin{align*} \nonumber
\pNew=&f^*(\pStrong,\pWeak,\pPred)\\ 
=&  \proj \left( \gradGinv \big( \gradG(\pStrong) + \gradG(\pWeak) - \gradG(\pPred) \big) \right).
\end{align*}
\end{definition}

\begin{algorithm}[ht]\begin{small}
\caption{Inference with Gradient Shift}
\label{alg:gradshift_decode}
\begin{algorithmic}
\STATE \textbf{Input:} Target model $\pStrong$, helper $\pWeak$, predictor $\pPred$, prompt $x$
\STATE \textbf{Output:} Generated sequence $y$
\STATE $y \leftarrow \langle\ \rangle$
\REPEAT
  \STATE $p_{\text{new}} \leftarrow f^*(\pStrong(\cdot \mid x,y), \pWeak(\cdot \mid x,y), \pPred(\cdot \mid x,y))$
  \STATE $t \leftarrow \textsc{SampleNextToken}(p_{\text{new}})$
  \STATE $y \leftarrow y \,\Vert\, t$
\UNTIL{$t = \textsc{EOT}$}
\STATE \textbf{return} $y$
\end{algorithmic}
\end{small}\end{algorithm}

We show that Gradient Shift decreases the expected loss by at least $\breg(\pWeak,\pPred)$, with a tight bound that implies optimality.

\subsection{Examples.}

We illustrate this gradient shift for standard loss functions.

\textbf{1. Quadratic Loss.} 
With $G(\mathbf{p}) = \|\mathbf{p}\|_2^2$, the gradients are linear, and $\proj$ becomes the standard Euclidean projection. The update is \textbf{additive}:
\begin{equation*}
    \pNew_{\text{add}}(\state) = (\pStrong(\state) + \pWeak(\state) - \pPred(\state) - \tau)_+,
\end{equation*}
where $x_+ = \max(0, x)$ and $\tau$ is the unique threshold such that $\sum_\state \pNew(\state) = 1$. The predictor simply adds the calibration residual $(\pWeak - \pPred)$ to the target prediction $\pStrong$, and projects the result back to the simplex if needed.

\textbf{2. Cross Entropy Loss.} 
Cross Entropy is the most widely used loss in Machine Learning. With $G(\mathbf{p}) = \sum_i p_i \ln p_i$ (negative entropy), the dual space corresponds to log-probabilities (logits). The update becomes \textbf{multiplicative}:
\begin{equation*}
    \pNew_{\text{mult}}(\state) = \frac{1}{Z} \frac{\pStrong(\state) \cdot \pWeak(\state)}{\pPred(\state)},
\end{equation*}
where $Z$ is the normalization constant. This is a geometric mixture of predictors: we scale the target probability $\pStrong$ by the likelihood ratio $\pWeak / \pPred$ and then renormalize.

This method has also been independently proposed as a heuristic in prior work, such as DExperts for output attribute control and Weak-to-Strong for LLM jailbreaking~\cite{Liu2021DExpertsDC,zhao2025weak}. We prove that this heuristic is optimal under our objective.

While the multiplicative form is widely used, it can be ill-conditioned when probabilities are tiny and noisy, whereas the additive form is often overly conservative. We therefore derive a hybrid strategy that combines the multiplicative form and additive form.

\paragraph{A Hybrid Strategy} 
This hybrid strategy uses geometric scaling in the suppression regime ($\pWeak < \pPred$) and an additive residual correction for amplification:
\begin{equation*}
    \pNew_{\text{hy}}(\state) = 
    \begin{cases} 
        \pStrong(\state) \cdot \pWeak(\state) / \pPred(\state) & \hspace{-6pt}\pWeak(\state) < \pPred(\state), \\
        \pStrong(\state) + \epsilon \left( \pWeak(\state) - \pPred(\state) \right) & \hspace{6pt}\text{otherwise,}
    \end{cases}
\end{equation*}
where $\epsilon$ is a constant chosen so that $\sum_\state \pNew_{\text{hy}}(\state) = 1$.

We show that this rule yields a prediction \textit{calibrated} with respect to $\statervRaw$ and is \textit{optimal} in the binary setting under Cross Entropy loss (see \Cref{coro}). For multiclass next-token prediction, we use hybrid as a stability-motivated variant of the multiplicative rule, consistent with the relaxation in \Cref{thm:relaxed}: it matches the multiplicative update when $\pWeak(\state) < \pPred(\state)$, while avoiding unstable amplification in low-probability regions via an additive correction and renormalization. We justify this design empirically across models and benchmarks in \Cref{sec:experiments}.

\begin{figure}[t]
    \centering
    \resizebox{1.0\columnwidth}{!}{
        \begin{tikzpicture}[>=Stealth, font=\sffamily, line width=0.6pt]

    \definecolor{colTarget}{RGB}{115, 169, 156}  
    \definecolor{colHelper}{RGB}{219, 109, 109}  
    \definecolor{colPred}{RGB}{214, 166, 101}    
    \definecolor{colResult}{RGB}{66, 114, 165}   
    \definecolor{colGray}{RGB}{100, 100, 100}

    \coordinate (A) at (1.5, 2.6); 
    \coordinate (B) at (0, 0);     
    \coordinate (C) at (3, 0);     

    \draw[thick, fill=gray!5] (A) -- (B) -- (C) -- cycle;
    \node[anchor=south] at (1.5, 2.7) {\scriptsize \textbf{Probability Simplex}};

    \coordinate (pt) at (1, 0.9);   
    \coordinate (ph) at (1.5, 1.8);   
    \coordinate (pth) at (2, 0.6);  
    \coordinate (pstar) at (1.2, 1.89); 

    \fill[colTarget] (pt) circle (2.5pt) node[left, black, font=\tiny] {$p_t$};
    \fill[colHelper] (ph) circle (2.5pt) node[above, black, font=\tiny] {$p_h$};
    \fill[colPred] (pth) circle (2.5pt) node[right, black, font=\tiny] {$p_{t|h}$};
    \node[star, star points=5, star point ratio=2.25, fill=colResult, inner sep=1.2pt] at (pstar) {};
    \node[below, black, font=\tiny, yshift=-2pt] at (pstar) {$p^*$};

    \begin{scope}[xshift=4.5cm, yshift=0cm]
        
        \node[anchor=south] at (1.5, 2.6) {\scriptsize \textbf{Dual/Gradient Space}};

        \coordinate (Origin) at (0,0);
        \coordinate (v_pt) at (0.4, 1.0);   
        \coordinate (v_ph) at (1.2, 2.0);   
        \coordinate (v_pth) at (2.2, 0.8);  
        
        \coordinate (v_star) at ($(v_pt) + (v_ph) - (v_pth)$);

        \draw[->, thick, colTarget] (Origin) -- (v_pt) node[left, font=\tiny] {$\nabla G(p_t)$};
        \draw[->, thick, colHelper] (Origin) -- (v_ph) node[above, font=\tiny] {$\nabla G(p_h)$};
        \draw[->, thick, colPred] (Origin) -- (v_pth) node[right, font=\tiny] {$\nabla G(p_{t|h})$};
        \draw[->, thick, colResult] (Origin) -- (v_star);

        \draw[dashed, thick, brown] (v_pth) -- (v_ph) node[midway, right, font=\tiny, align=left] {Alignment\\Shift Vector};
        
        \draw[dashed, thick, brown, ->] (v_pt) -- (v_star);

    \end{scope}

    
    \draw[->, thick, gray] (2.0, 2.2) to[bend left=30] node[midway, above, font=\tiny] {$\nabla G$} (5.0, 2.5);

    \draw[->, thick, gray] (5.0, -0.2) to[bend left=30] node[midway, below, font=\tiny] {$(\nabla G)^{-1} \ \& \ \Pi_{\Delta}$} (2.0, -0.2);

\end{tikzpicture}
    }
    \caption{\textbf{Overview of the Gradient Shift method}}
    \label{fig:gradientshift}
\end{figure}
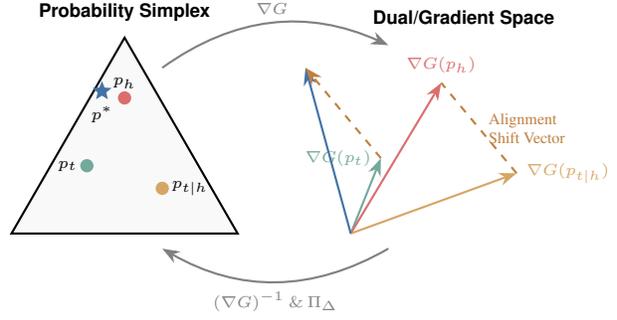

\section{Theoretical Guarantees}

We now present our main result: Gradient Shift rule is optimal, and the corresponding $\pNewrv$ strictly dominates the target model with respect to the pre-alignment data distribution.

\begin{restatable}[Optimality of Gradient Shift]{theorem}{thmmain}
\label{thm:main}
Let $\score$ be a strictly proper loss induced by $G$ satisfying the regularity conditions in \Cref{def:reg}. For any joint distribution $\dist$ consistent with \Cref{assume}, the gradient shift rule $\pNewrv=f^*(\pStrongrv,\pWeakrv,\pPredrv)$ (\Cref{def:grashift}) strictly dominates the target model $\pStrongrv$. Specifically,
\begin{equation*}
    \E \left[ \score(\pNewrv, \statervRaw) - \score(\pStrongrv, \statervRaw) \right] \geq \E \left[ \breg(\pWeakrv, \pPredrv) \right].
\end{equation*}
In fact, the following stronger conditional result holds pointwise. Conditioning on any fixed realization of the helper $\pWeakrv=\pWeak$ and predictor $\pPredrv=\pPred$,
\begin{equation*}\label{eq:strong_result}
    \E \left[ \score(\pNewrv, \statervRaw) - \score(\pStrongrv, \statervRaw)\right] \geq \breg(\pWeak, \pPred).
\end{equation*}
Furthermore, this lower bound is tight: for any given $(\pWeak, \pPred)$, there exists a joint distribution consistent with the assumptions such that the inequality holds with equality for any aggregation strategy $f(\cdot)$. Thus, the gradient shift rule $f^*$ is optimal, with maximal $v(f^*)=\max_f v(f)=\E[\breg(\pWeakrv, \pPredrv)]$.
\end{restatable}

Consequently, the performance improvement is strictly positive provided there is a divergence between the helper model's output distribution and its prediction of the target model. Furthermore, we \emph{relax} our initial assumptions to evaluate robustness when predictor and helper are approximations rather than exact conditional expectations. In such cases, we show that while a residual error term is introduced, the gradient shift rule still yields a net benefit that scales with the accuracy of these approximations (\Cref{thm:relaxed}).

\begin{restatable}{corollary}{cormain}\label{coro}
Under the assumptions of \Cref{thm:main}, the following rules are optimal for their respective loss functions: the \emph{additive shift} is optimal under quadratic loss; the \emph{multiplicative shift} is optimal under cross entropy loss; and the \emph{hybrid strategy} is optimal under cross entropy loss in the binary setting $|\statesp|=2$.
\end{restatable}

\paragraph{Proof Sketch} We focus on proving the conditional version \eqref{eq:strong_result}, since taking the expectation of this result over $\pWeakrv$ and $\pStrongrv$ yields the main theorem. We formalize the problem as a minimax game. The adversary chooses a joint distribution over the target model's prediction $\pStrongrv$ and the pre-aligned $\statervRaw$, subject to $\E_{\dist}[\pStrongrv]=\pPred$ and $\E_{\dist}[\statervRawvec]=\pWeak$. Here $\statervRawvec$ denotes the one-hot vector $\delta_{\statervRaw}$. The player chooses an aggregator $f$. The objective is
\[
\E_{\dist}\!\left[\score(\pStrongrv, \statervRaw) - \score\!\left(f(\pStrongrv,\pWeak,\pPred), \statervRaw\right)\right].
\]
The player aims to maximize the performance gain, while the adversary aims to minimize it.

We shorthand $\stateRaw$ by $\state$, $\pStrong$ by $\pvec$, $\pPred$ by $\bar{\pvec}$, and $\pWeak$ by $\bar{\statevec}$. Since $\pWeak$ and $\pPred$ are fixed parameters, we drop them from the input of $f$ for clarity.  With this notation, the minimax game becomes:
\begin{align*}
    V^* = \sup_f \inf_{\substack{\dist(\mathbf{P},Y) \\ \E_{\dist}[\mathbf{P}]=\bar{\mathbf{p}}, \, \E_{\dist}[\statervvec]=\bar{\statevec}}} 
    \E_{\dist}\!\left[\score(\pvec, \state) - \score\!\left(f(\pvec), \state\right)\right].
\end{align*}
We first identify an upper bound by noting that if the adversary (choosing $\dist$) plays first, they can minimize the potential gain by removing all randomness and fixing the data at the means $(\bar{\pvec}, \bar{\statevec})$, due to convexity of the loss. We then show that this value is achievable using the gradient-shift rule without projection (\Cref{def:grashift}). Finally, we apply the Generalized Pythagorean Theorem to show that projecting this strategy onto the probability simplex ensures feasibility while preserving optimality.

For the hybrid strategy, it suffices to show that the minimal gain equals the KL divergence $\kl(\pWeak, \pPred)$ in the binary setting. We use two key lemmas: (1) the expectation of the hybrid aggregator equals the prediction, i.e., $\E[\mathbf{p}_{\text{hy}}]=\pPred$; and (2) the performance gain is convex under the hybrid aggregator. As a result, the worst case occurs at a degenerate distribution, where the gain equals $\kl(\pWeak, \pPred)$.

\section{Experimental Setup}
\label{sec:experiments}

We evaluate our jailbreaking solutions with two complementary experiments to test generality across settings. We vary the target task and helper-model choice to assess whether our approach extends beyond a single benchmark under a unified framework: (1) a fixed-threat-model red-teaming study comparing inference-time aggregation rules under a shared protocol; and (2) a math jailbreaking experiment following the established jailbreak-tax methodology with comparisons to multiple prior jailbreak baselines. Implementation details are provided in Appendix~\ref{sec:expapp}.

Our theoretical framework optimizes a proper loss with respect to the pre-aligned distribution $\statervRaw$, which is not directly observable. We therefore report Attack Success Rate (ASR) and Jailbreak Tax. Higher ASR and lower jailbreak tax indicates that the aggregation rules recover properties of $\statervRaw$ while preserving the model's capabilities.

\subsection{Red-teaming Jailbreaking}
\label{sec:redteam-jailbreak}

In this experiment, we evaluate red-teaming jailbreaks on harmful benchmark prompts, measuring both ASR and the response harmfulness, since successful jailbreaks can still be low-quality or non-actionable.

\paragraph{Models and Settings.}
We evaluate two target-model families: \textsf{Llama-3.3-70B-Instruct} as a standard LLM, and \textsf{gpt-oss-120b} as a Large Reasoning Model (LRM)~\cite{grattafiori2024llama,agarwal2025gpt}. Both are recent instruction-tuned models with stronger post-training safety alignment. In particular, \textsf{gpt-oss} uses advanced alignment techniques such as deliberative alignment and instruction hierarchy~\cite{guan2024deliberative,wallace2024instruction}.

Our gradient-shift method uses a helper and a predictor model. For \textsf{Llama-3.3-70B-Instruct}, we use \textsf{Orenguteng--Llama-3.1-8B-Lexi-Uncensored-V2} (helper) and \textsf{Llama-3.1-8B-Instruct} (predictor)~\cite{orenguteng2024llama3.1lexi}. For \textsf{gpt-oss-120b}, we use \textsf{ArliAI--gpt-oss-20b-Derestricted} (helper) and \textsf{gpt-oss-20b} (predictor)~\cite{gptoss20bderestricted2025}.

Prior jailbreak methods assume different threat models and attack budgets (e.g., prompt-optimization, multi-turn interaction, external judges), making direct ASR/harmfulness comparisons under a single standardized protocol misleading. 
Therefore, in this experiment we focus on controlled comparisons within the inference-time aggregation family. Broader comparisons against prior jailbreak methods are reported under the jailbreak-tax math evaluation (\Cref{fig:gsm8k}). As a baseline, we use \textbf{Weak-to-Strong}~\cite{zhao2025weak}, which applies logit manipulation only to the first $k$ generated tokens and scales its magnitude by an amplification factor $\alpha \in [1,\infty)$ to improve ASR. When $\alpha=1$ and $k=\infty$, Weak-to-Strong matches our multiplicative aggregator.

\paragraph{Datasets.} We evaluate on two widely adopted red-teaming jailbreaking datasets, \textbf{HarmBench} and \textbf{StrongREJECT}~\cite{mazeika2024harmbench,souly2024strongreject}. HarmBench is a large-scale evaluation suite for automated red teaming and robust refusal, while StrongREJECT provides forbidden prompts with rubric-based evaluators that score how well a response contains disallowed harmful content. We use all 200 prompts from HarmBench and a 200-prompt subset from StrongREJECT for evaluation.

\paragraph{Metrics.}
We evaluate both attack success and harmfulness of the generated outputs.

\begin{itemize}[leftmargin=12pt]
    \vspace{-1.0em}
    \item \textbf{Attack Success Rate (ASR).} We report ASR using the HarmBench ASR evaluator (a transformer-based classifier)~\cite{mazeika2024harmbench}. ASR is the fraction of prompts where the target model is judged to have produced the requested harmful response. However, ASR alone does not capture the harmfulness.
    \vspace{-0.5em}
    \item \textbf{Harmfulness level.} We score harmfulness using the \textbf{StrongREJECT rubric evaluator}, which uses a remotely hosted frontier model to score responses according to a rubric~\cite{souly2024strongreject}. Due to space constraints, we defer additional widely used automated judges, \textbf{MD-Judge}~\cite{Li2024SALADBenchAH} and \textbf{HarmScore}~\cite{chan2025speak}, to Appendix~\ref{sec:expapp}.
    \vspace{-0.5em}
\end{itemize}

\subsection{Math Jailbreaking}
\label{sec:math-jailbreak}

In this experiment, we evaluate math jailbreaking under the jailbreak tax methodology~\cite{nikolic2025jailbreak}, which uses ground-truth math benchmarks by aligning models to refuse math and then measuring how much verifiable utility is recovered after jailbreaking. With ground-truth solutions, we directly measure both success rate and answer accuracy.

\paragraph{Models and Settings.}
We jailbreak a target model explicitly LoRA-tuned to refuse math: \textsf{ethz-spylab--Llama-3.1-70B-Instruct\_refuse\_math}. Our gradient-shift method uses \textsf{meta-llama--Llama-3.1-8B-Instruct} as the helper and \textsf{ethz-spylab--Llama-3.1-8B-Instruct\_refuse\_math} as the predictor~\cite{nikolic2025jailbreak}. For comparison, we reprint the results of prior jailbreaks reported by \citet{nikolic2025jailbreak} and evaluate Weak-to-Strong as an extra baseline.

\paragraph{Datasets.}
We evaluate on two math benchmarks with ground-truth answers: \textbf{GSM8K} and \textbf{Hendrycks-MATH}~\cite{cobbe2021gsm8k,hendrycksmath2021}. GSM8K consists of grade-school word problems, while MATH contains competition-style problems across difficulty levels. We report results on GSM8K and MATH (levels 1, 3, 5), using 300-example subsets for each setting.

\paragraph{Metrics.}
We decompose jailbreak performance into \textbf{answering} and \textbf{correctness}:
\begin{itemize}[leftmargin=12pt]
    \vspace{-1.0em}
    \item \textbf{Success rate.} Fraction of prompts where the refuse-math target gives a final answer (i.e., does not refuse), regardless of correctness.
    \vspace{-0.5em}
    \item \textbf{Correct rate.} Fraction of prompts where the target gives a final answer and it is correct.
    \vspace{-1.0em}
\end{itemize}

We compute jailbreak tax as the fraction of baseline utility not recovered after jailbreaking. Let $\mathrm{BenchAcc}$ denote the accuracy of the unaligned benchmark model (\textsf{Llama-3.1-70B-Instruct}) and $\mathrm{TrueAcc}$ the accuracy conditioned on successful jailbreaks. Following \citet{nikolic2025jailbreak}, jailbreak tax is defined as \[\mathrm{JTax}=1- \mathrm{TrueAcc} / \mathrm{BenchAcc},\] where larger values indicate less recovered pre-alignment math ability.

\begin{figure*}[!ht]
\centering
\begin{minipage}[t]{0.642\textwidth}
  \centering
  \begin{subfigure}[t]{0.498\textwidth}
    \centering
    \includegraphics[width=\linewidth]{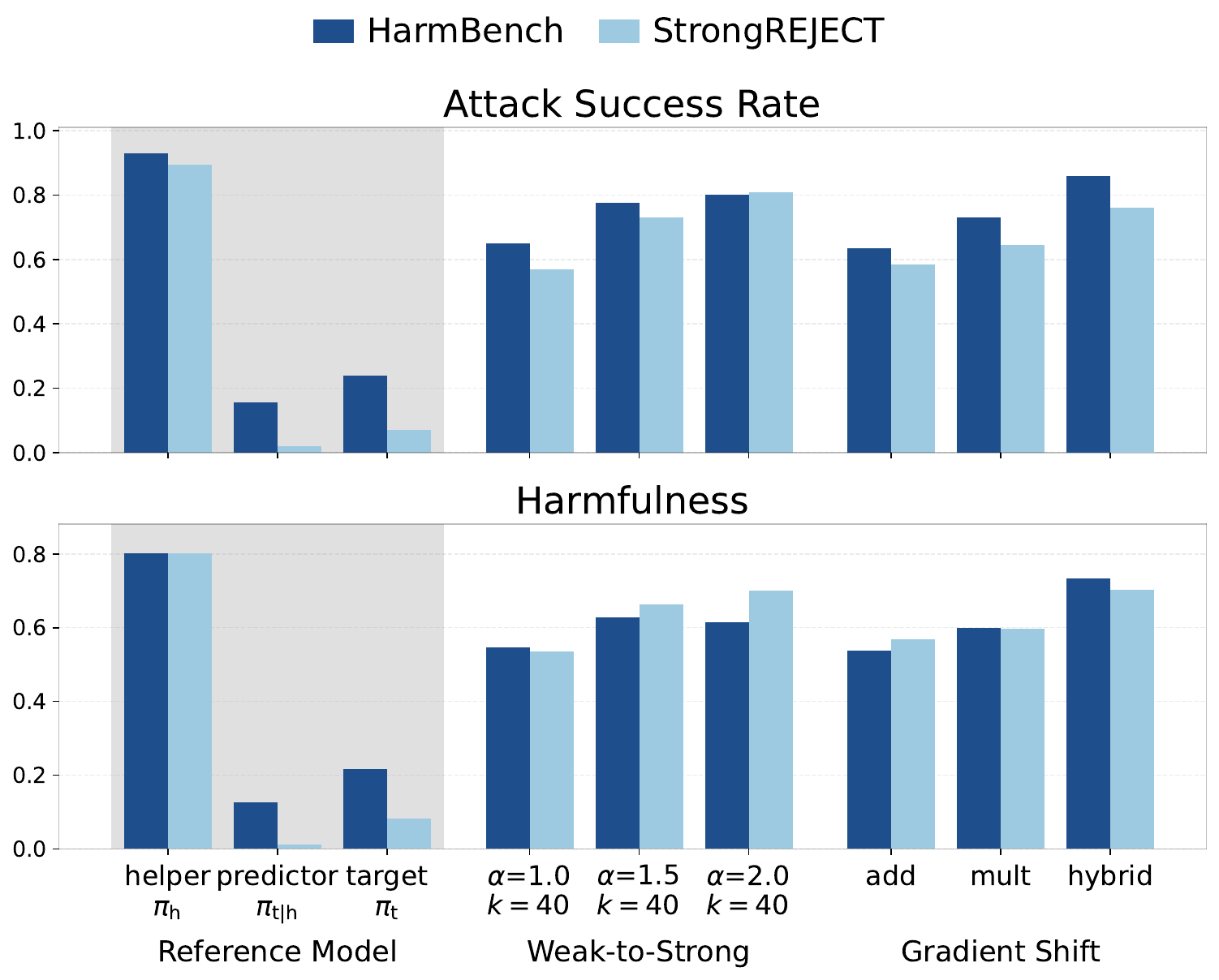}
    \caption{Jailbreaking \textsf{Llama-3.3-70B-Instruct}}
    \label{fig:redteam_llm}
  \end{subfigure}\hfill
  \begin{subfigure}[t]{0.498\textwidth}
    \centering
    \includegraphics[width=\linewidth]{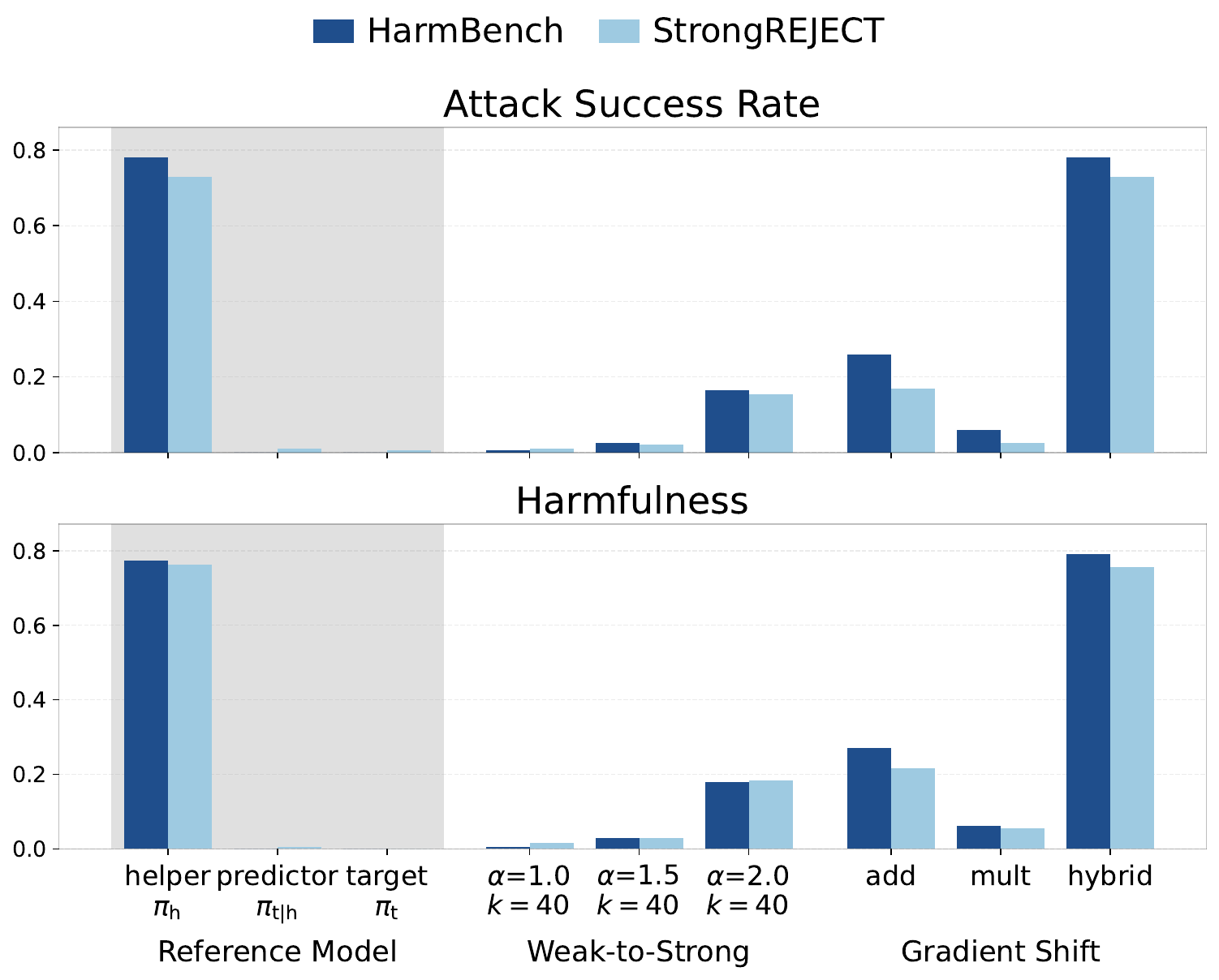}
    \caption{Jailbreaking \textsf{gpt-oss-120b}}
    \label{fig:redteam_lrm}
  \end{subfigure}

  \captionof{figure}{\textbf{Red-teaming jailbreaking LLM and LRM}: On both HarmBench and StrongREJECT datasets, our Hybrid Gradient Shift successfully jailbreaks both standard LLM and LRM, outperforming Weak-to-Strong and other Gradient Shift variants.}
  \label{fig:redteam_main}
\end{minipage}\hfill
\begin{minipage}[t]{0.35\textwidth}
  \centering
  \begin{tikzpicture}
    \node[anchor=south west, inner sep=0] (base) at (0,0)
      {\includegraphics[width=0.95\linewidth]{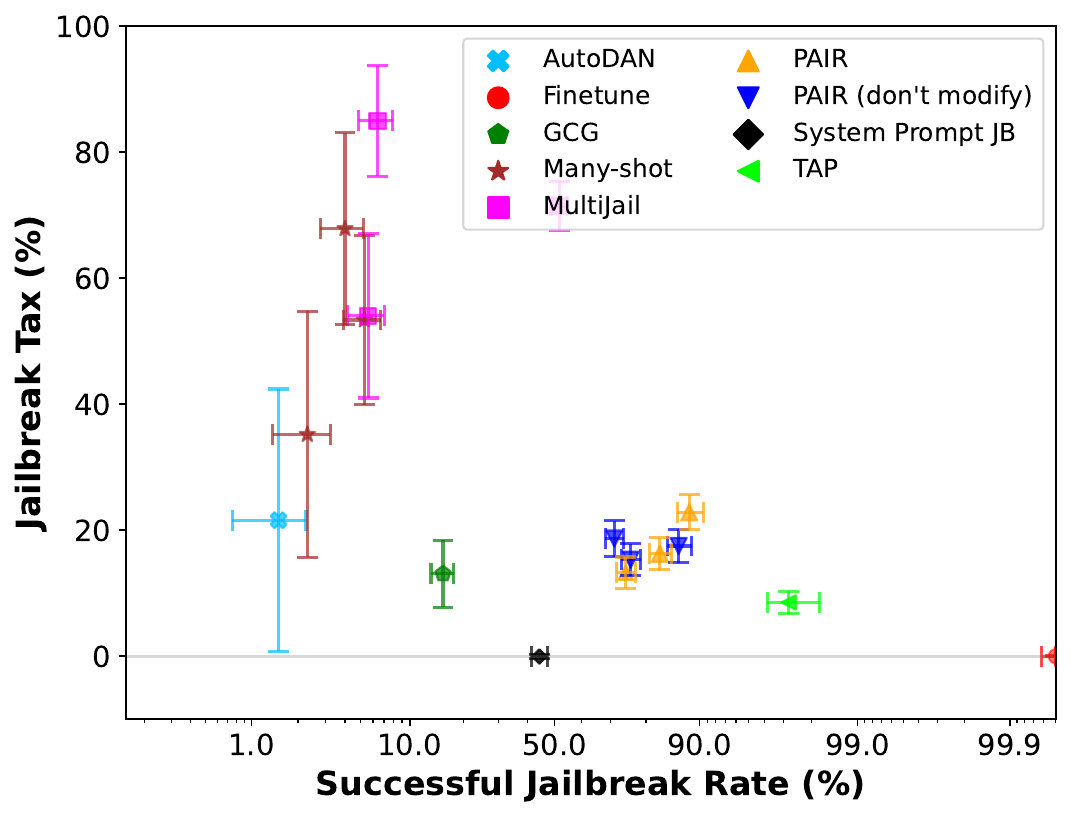}};
    \node[anchor=south west, inner sep=0] at (base.south west)
      {\includegraphics[width=0.95\linewidth]{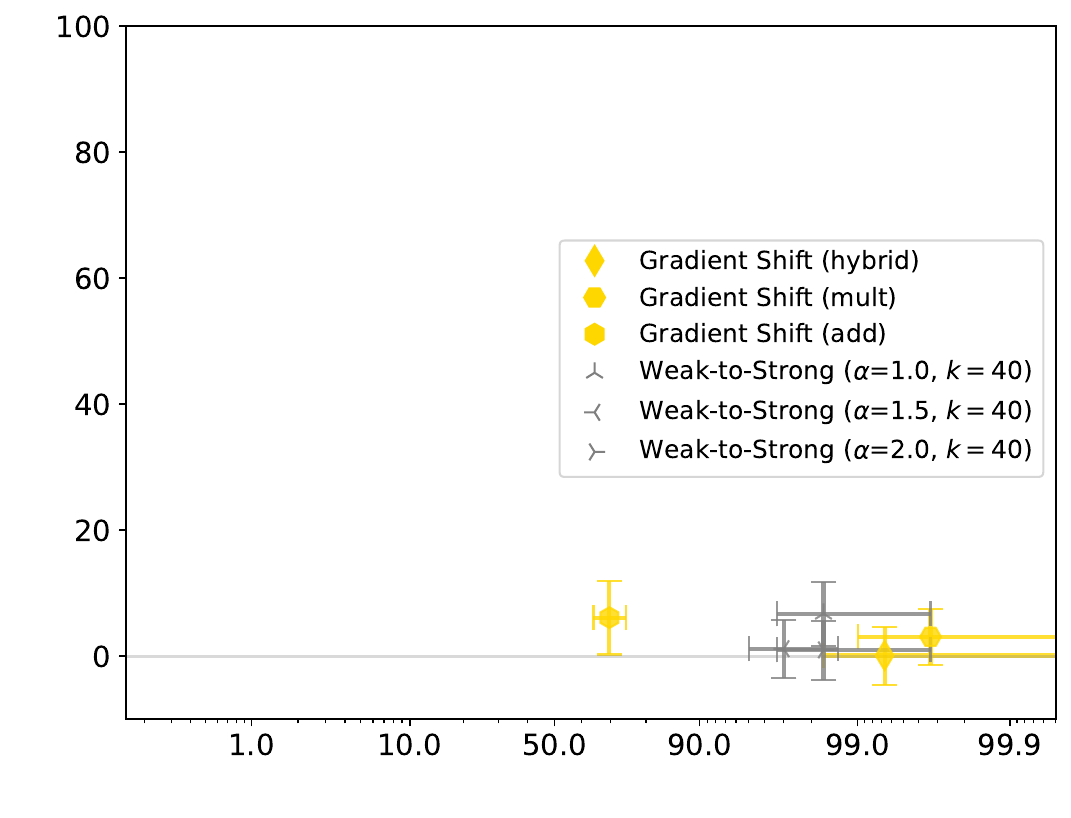}};
  \end{tikzpicture}

  \captionof{figure}{\textbf{Jailbreak success rate and jailbreak tax on GSM8K}: We overlay our evaluation results for Gradient Shift and Weak-to-Strong on a plot reprinted from \citet{nikolic2025jailbreak}. Other methods are from the same reference under the identical evaluation protocol.}
  \label{fig:gsm8k}
\end{minipage}
\end{figure*}

\begin{table*}[!ht]
  \centering

  \setlength{\tabcolsep}{1pt}
  \renewcommand{\arraystretch}{0.92}
  \begin{small}\begin{tabular}{lccc||ccc||ccc||ccc}
    \toprule
    & \multicolumn{3}{c}{GSM8K} & \multicolumn{3}{c}{MATH (level 1)} & \multicolumn{3}{c}{MATH (level 3)} & \multicolumn{3}{c}{MATH (level 5)} \\
    \cmidrule(lr){2-4}\cmidrule(lr){5-7}\cmidrule(lr){8-10}\cmidrule(lr){11-13}
    Method
    & Correct$\uparrow$ & Success$\uparrow$ & JTax$\downarrow$
    & Correct$\uparrow$ & Success$\uparrow$ & JTax$\downarrow$
    & Correct$\uparrow$ & Success$\uparrow$ & JTax$\downarrow$
    & Correct$\uparrow$ & Success$\uparrow$ & JTax$\downarrow$ \\
    \midrule
    Gradient Shift (hybrid)
    & \textbf{0.88} & 0.99 & \textbf{0.00}
    & \textbf{0.79} & \textbf{0.97} & -0.16
    & 0.55 & 0.92 & -0.15
    & 0.19 & 0.80 & -0.08 \\
    Gradient Shift (mult)
    & 0.86 & \textbf{1.00} & 0.03
    & 0.73 & 0.96 & -0.09
    & \textbf{0.56} & \textbf{0.96} & -0.12
    & \textbf{0.26} & \textbf{0.84} & \textbf{-0.41} \\
    Gradient Shift (add)
    & 0.58 & 0.70 & 0.06
    & 0.57 & 0.74 & -0.10
    & 0.38 & 0.66 & -0.11
    & 0.11 & 0.51 & 0.02 \\
    Weak-to-Strong ($\alpha$=1.0, $k$=40)
    & 0.81 & 0.98 & 0.07
    & 0.70 & 0.89 & -0.12
    & 0.51 & 0.80 & \textbf{-0.23}
    & 0.13 & 0.66 & 0.10 \\
    Weak-to-Strong ($\alpha$=1.5, $k$=40)
    & 0.85 & 0.97 & 0.01
    & 0.69 & 0.90 & -0.10
    & 0.50 & 0.80 & -0.20
    & 0.15 & 0.68 & -0.00 \\
    Weak-to-Strong ($\alpha$=2.0, $k$=40)
    & 0.86 & 0.98 & 0.01
    & 0.72 & 0.88 & \textbf{-0.17}
    & 0.46 & 0.77 & -0.15
    & 0.11 & 0.74 & 0.32 \\
    \midrule
    Benchmark LLM ($\mathrm{BenchAcc}$)
    & 0.89 & N/A & N/A
    & 0.70 & N/A & N/A
    & 0.52 & N/A & N/A
    & 0.22 & N/A & N/A \\
    Helper LLM
    & 0.85 & N/A & N/A
    & 0.61 & N/A & N/A
    & 0.37 & N/A & N/A
    & 0.10 & N/A & N/A \\
    \bottomrule
  \end{tabular}\end{small}
  \caption{\textbf{Math jailbreaking results on GSM8K and MATH}: $\mathrm{JTax}=1-\frac{\mathrm{TrueAcc}}{\mathrm{BenchAcc}}$, where higher values indicate less recovered utility.}
  \label{tab:math}
\end{table*}

\section{Experimental Results}

\paragraph{Hybrid Gradient Shift breaks hardened safety defenses.}

\Cref{fig:redteam_main} summarizes red-teaming results on \textsf{Llama-3.3-70B-Instruct} (LLM) and the safety-hardened \textsf{gpt-oss-120b} (LRM), comparing Hybrid Gradient Shift with other gradient-shift variants and the Weak-to-Strong baseline using ASR and harmfulness level.
The key finding is that Hybrid Gradient Shift is the only approach that can effectively jailbreak \textsf{gpt-oss-120b}: its ASR reaches the helper-model ASR, while all other methods remain below 30\%.
On \textsf{Llama-3.3-70B-Instruct}, Hybrid Gradient Shift remains strong and generally competitive across both metrics.
Together, these results show that Hybrid Gradient Shift is crucial for overcoming the hardened defenses of \textsf{gpt-oss-120b} without sacrificing performance on standard LLM targets.

\paragraph{Gradient Shift achieves substantial math success rate.}
\Cref{tab:math} shows that Gradient Shift variants attain consistently high non-refusal rates on the refuse-math target across GSM8K and MATH.
Hybrid and multiplicative Gradient Shift are near-perfect on GSM8K and stay above 92\% on MATH levels 1 and 3, with a moderate drop at level 5, since harder problems require longer solutions that more often hit the \texttt{max\_tokens} limit and get truncated, thus not counting as success.
By contrast, additive Gradient Shift is markedly less reliable, and Weak-to-Strong is consistently lower than Hybrid/multiplicative, especially on harder MATH settings.
Overall, Gradient Shift robustly recovers answering behavior under refusal-style alignment.

\paragraph{Gradient Shift incurs no jailbreak tax.}
Beyond eliciting answers, Gradient Shift preserves the underlying math capability of the model. As reported in \Cref{tab:math}, Hybrid matches the benchmark accuracy on both GSM8K and MATH while incurring zero jailbreak tax\footnote{The negative taxes on harder MATH problems are a selection effect: harder problems require longer reasoning and are more likely to be refused, so conditioning on successful jailbreaks filters out disproportionately difficult instances and increases $\mathrm{TrueAcc}$.}.
Notably, this utility is not explained by outputting the helper: the helper's standalone accuracy is substantially lower than both Gradient Shift and the benchmark model.
\Cref{fig:gsm8k} places our GSM8K results from \Cref{tab:math} into the broader jailbreak-tax comparison of \citet{nikolic2025jailbreak}, showing that Gradient Shift attains near-saturated success at zero jailbreak tax, approaching the performance of directly fine-tuning the target model.

\section{Discussion}

\paragraph{Limitations}
Our approach requires a weak, unaligned model that shares the target's tokenizer. This requirement is often feasible in practice since small models are inexpensive to fine-tune. While the strict calibration assumptions are idealized, our relaxation analysis and empirical results suggest they are not required for practical success. The strong performance of our methods suggests that standard instruction-tuned models are ``calibrated enough'' to serve as proxies in the gradient space, effectively acting as noisy but informative estimators of the pre-aligned distribution.

\paragraph{Broader Impact}
While we study jailbreaking, the theoretical framework is more general: it provides a way to shift a model's behavior using another model as a reference. For example, the same idea can be used in the opposite direction for jailbreak defense, by shifting an unsafe model toward a safer distribution. It can also be used for transferring skills between two models or combining strengths, e.g., the reasoning ability of a large model with the domain knowledge of a small specialist model. Among these possible uses, jailbreaking is the easiest downstream application, because it is the most tolerant of prediction noise: success is often detectable even when the predictor is imperfect.

\section{Impact Statement}

This work studies inference-time jailbreaking to better understand and measure vulnerabilities in current LLM safety mechanisms. Although such techniques could be misused, we present them to support defensive research: stronger threat models help developers identify failure modes and improve guardrails. More broadly, our theoretical framework is not limited to jailbreaking and can also support defensive use cases. As LLMs become more capable and are deployed more widely, understanding how safety mechanisms can fail, and how they may be exploited, becomes increasingly important. We encourage responsible disclosure and use, and hope these findings will help drive progress toward robust and scalable security defenses.

\bibliography{ref}
\bibliographystyle{icml2026}

\newpage
\onecolumn
\appendix

\section{More Examples for Gradient Shift}\label{sec:power}

\textbf{Power Loss.}
For the family of divergences generated by $G(\pvec) = \frac{1}{\beta(\beta-1)} \sum p_i^\beta$ where $\beta \in (1, 2]$, the dual space transforms probabilities by the power $\beta-1$. The update becomes a \textbf{power-space shift}: for all $\state$, $\pNew(\state)$ is
\begin{equation*}\label{eq:power-shift}
    \left( \pStrong(\state)^{\beta-1} + \pWeak(\state)^{\beta-1} - \pPred(\state)^{\beta-1}-\tau \right)_+^{\frac{1}{\beta-1}},
\end{equation*}
where $x_+=\max(0,x)$ and $\tau$ is the unique threshold such that $\sum_y\pNew(y)=1$. This family provides a continuum between the extremes: as $\beta \to 1$, the update approaches the multiplicative geometric mixture (Cross Entropy Loss), and at $\beta = 2$, it recovers the additive arithmetic shift (Quadratic Loss).

\section{Omitted Proofs}

\newcommand{\gap}{\nabla}

\thmmain*

\begin{proof}[Proof of \Cref{thm:main}]

We shorthand $\stateRaw$ by $\state$, $\pStrong$ by $\pvec$, $\pPred$ by $\bar{\pvec}$, and $\pWeak$ by $\bar{\statevec}$. Since $\pWeak$ and $\pPred$ are fixed parameters, we drop them from the input of $f$ for clarity. 

We analyze the minimax game between a predictor (choosing $f$) and an adversary (choosing distribution $\dist$) subject to moment constraints $\E[P]=\bar{\pvec}$ and $\E[\statervvec]=\bar{\statevec}$. We seek the value $V^*$:
\begin{align} \label{eq:obj_divergence}
    V^* = \sup_f \inf_{\substack{\dist(\prvvec,\statervvec) \\ \E[\prvvec]=\bar{\pvec}, \, \E[\statervvec]=\bar{\statevec}}} 
    \E_{\dist}[\breg(\statervvec, \prvvec) - \breg(\statervvec, f(\prvvec))]
\end{align}

\noindent \textbf{Step 1: Upper Bound.} 
We first determine the maximum possible value the predictor can hope for by analyzing the game where the adversary plays first. By the minimax inequality (swapping the order of sup and inf):
\begin{align}
    V^* \leq \inf_{\dist} \sup_f \E_{\dist}[\breg(\statervvec, \prvvec) - \breg(\statervvec, f(\prvvec))]
\end{align}
For any fixed $\dist$, the optimal predictor $f$ is the conditional expectation $f(\prvvec) = \E[\statervvec|\prvvec]$. Substituting this back into the objective yields the intrinsic uncertainty:
\begin{align}
    \E_{\dist}[\breg(\E[\statervvec|\prvvec], \prvvec)]
\end{align}
The adversary minimizes this quantity. Since $\breg$ is jointly convex, by Jensen's inequality, the minimum is achieved when the random variables are replaced by their means (i.e., a deterministic distribution where $\prvvec=\bar{\pvec}$ and $\statervvec=\bar{\statevec}$):
\begin{align}
    \inf_{\dist}\E_{\dist}[\breg(\E[\statervvec|\prvvec], \prvvec)] = \breg(\bar{\statevec}, \bar{\pvec})
\end{align}
Thus, we establish the upper bound:
\begin{align} \label{eq:upper_bound}
    V^* \leq \breg(\bar{\statevec}, \bar{\pvec})
\end{align}

\noindent \textbf{Step 2: Achievability.}
We now show the predictor can guarantee this value by choosing a specific function. Let $\gap = \gradG(\bar{\statevec}) - \gradG(\bar{\pvec})$. We define the predictor $f$ such that:
\begin{align} \label{eq:link_function}
    \gradG(f(\pvec)) = \gradG(\pvec) + \gap
\end{align}
Using the three-point identity for Bregman divergences, the objective expands to:
\begin{align}
    \breg(\statervvec, \prvvec) - \breg(\statervvec, f(\prvvec)) = \breg(f(\prvvec), \prvvec) + \langle \gradG(f(\prvvec)) - \gradG(\prvvec), Y - f(\prvvec) \rangle
\end{align}
Due to the definition of $f$, the gradient difference becomes the constant vector $\gap$. Taking the expectation over any valid distribution $\dist$:
\begin{align}
    \E[\dots] &= \E\left[ \breg(f(\prvvec), \prvvec) + \langle \gap, Y - f(\prvvec) \rangle \right] \nonumber \\
    &= \E\left[ \breg(f(\prvvec), \prvvec) - \langle \gap, f(\prvvec) \rangle \right] + \langle \gap, \bar{\statevec} \rangle \label{eq:expectation_step}
\end{align}
Consider the function $h(\mathbf{z}) = \breg(\mathbf{z}, \bar{\pvec}) - \langle \gap, \mathbf{z} \rangle=\breg(\mathbf{z}, \bar{\pvec}) - \langle \gradG(\bar{\statevec}) - \gradG(\bar{\pvec}), \mathbf{z} \rangle$. This function is convex and minimized at $\mathbf{z}=\bar{\statevec}$. Therefore, by Jensen's inequality applied to the term inside the expectation in \eqref{eq:expectation_step}:
\begin{align}
    \E[\dots] & \geq  \breg(\E[f(\prvvec)], \bar{\pvec}) - \langle \gap, \E[f(\prvvec)] \rangle  + \langle \gap, \bar{\statevec} \rangle\nonumber \\
    &\geq h(\bar{\statevec}) + \langle \gap, \bar{\statevec} \rangle \nonumber \\
    &= (\breg(\bar{\statevec}, \bar{\pvec}) - \langle \gap, \bar{\statevec} \rangle) + \langle \gap, \bar{\statevec} \rangle \nonumber \\
    &= \breg(\bar{\statevec}, \bar{\pvec})
\end{align}
This confirms that the predictor guarantees at least $\breg(\bar{\statevec}, \bar{\pvec})$. Combined with Step 1, $V^* = \breg(\bar{\statevec}, \bar{\pvec})$.

\noindent \textbf{Step 3: Projection.}
The predictor $f(\pvec)$ defined in Step 2 may output values outside the probability simplex $\Delta$. Let $f_{proj}(\pvec) = \proj(f(\pvec))$ be the Bregman projection of $f(\pvec)$ onto $\Delta$. By the Generalized Pthagorean Theorem, for any $\mathbf{q} \in \Delta$:
\begin{align}
    \breg(\mathbf{q}, f(\pvec)) \geq \breg(\mathbf{q}, f_{proj}(\pvec)) + \breg(f_{proj}(\pvec), f(\pvec)) \geq \breg(\mathbf{q}, f_{proj}(\pvec))
\end{align}
Thus, the projected predictor $f_{proj}(\pvec)$ yields a lower (better) loss than the unprojected $f(\pvec)$, ensuring the strategy remains optimal and valid within the simplex. $f_{proj}(\pvec)$ is precisely the gradient shift rule $f^*$. Thus, the gradient shift rule is optimal.
\end{proof}

\cormain*

\begin{proof}[Proof of \Cref{coro}]

The optimality of additive, multiplicative and power-space shift form follows directly from the main theorem. It remains to prove the optimality of the hybrid form in the binary case. 

We have shown that \begin{align}
    \sup_f \inf_{\substack{\dist(\pvec,Y)  \\ \E_{\dist}[\pvec]=\bar{\pvec}, \, \E_{\dist}[\statervvec]=\bar{\statevec}}} 
    \E_{\dist}[-\log(\pvec_\staterv) + \log(f(\pvec)_\staterv)]=\kl(\bar{\mathbf{y}},\bar{\pvec}).
\end{align} 

Thus we only need to prove for 
\begin{equation}
    \label{opt:1}
    f(\pvec)_s = 
    \begin{cases} 
        \frac{\pvec_s \cdot \bar{\mathbf{y}}_s}{\bar{\pvec}_s} & \text{if } \bar{\mathbf{y}}_s < \bar{\pvec}_s, \\
        \pvec_s + \epsilon \left( \bar{\mathbf{y}}_s - \bar{\pvec}_s \right) & \text{else},
    \end{cases}
\end{equation} where $\epsilon$ is a normalization constant,

\begin{align}
    \inf_{\substack{\dist(\pvec,Y)  \\ \E_{\dist}[\pvec]=\bar{\pvec}, \, \E_{\dist}[\statervvec]=\bar{\statevec}}} 
    \E_{\dist}[-\log(\pvec_\staterv) + \log(f(\pvec)_\staterv)]\ge \kl(\bar{\mathbf{y}},\bar{\pvec}).
\end{align} 

First we prove two lemmas.

\begin{lemma}
$f(\bar{\pvec})=\bar{\mathbf{y}}$.
\end{lemma}
\begin{proof}
    For $\bar{\mathbf{y}}_s < \bar{\pvec}_s$,
    \begin{align*}
    f(\bar{\pvec})=&\bar{\pvec}_s\cdot\bar{\mathbf{y}}_s/\bar{\pvec}_s\\
        =&\bar{\mathbf{y}}_s
    \end{align*}

    By summing up all $s$ in \cref{opt:1}, we have:

    $$1=\sum_{s:\bar{\mathbf{y}}_s < \bar{\pvec}_s}\bar{\mathbf{y}}_s+\sum_{s:\bar{\mathbf{y}}_s \geq \bar{\pvec}_s}(\epsilon(\bar{\mathbf{y}}_s-\bar{\pvec}_s)+\bar{\pvec}_s).$$

    Since $\sum_s \bar{\mathbf{y}}_s=1$, comparing with the above normalization implies $(\epsilon-1)\sum_{s:\bar{\mathbf{y}}_s\ge \bar{\pvec}_s}(\bar{\mathbf{y}}_s-\bar{\pvec}_s)=0$, hence $\epsilon=1$. By inserting $\epsilon$ we obtain that for $\bar{\mathbf{y}}_s \ge \bar{\pvec}_s$,
    \begin{align*}
    f(\bar{\pvec}_s)=\bar{\pvec}_s+\epsilon(\bar{\mathbf{y}}_s-\bar{\pvec}_s)=\bar{\mathbf{y}}_s
    \end{align*}

\end{proof}

\begin{lemma}
    For any $x,A,B>0$, let $h(x)=\log((Ax+B)/x)$, then $h(x)$ is convex.
\end{lemma}

\begin{proof}
    By direct calculation,

    \begin{align*}
    h''(x) = \frac{2ABx+B^2}{(x(Ax+B))^2}\ge0.
    \end{align*}

    Thus, $h(x)$ is convex.
\end{proof}

Then using these lemmas, 

\begin{align*}
    &\E_{\dist}[-\log(\pvec_\staterv) + \log(f(\pvec)_\staterv)]\\
    =&\E_{\dist}\sum_\state \Pr(Y=\state)\log\left(\frac{f(\pvec)_\state}{\pvec_\state}\right)\\
    =&\sum_{y:\bar{\mathbf{y}}_\state<\bar{\pvec}_\state}\bar{\mathbf{y}}_\state\log\left(\frac{\bar{\mathbf{y}}_\state}{\bar{\pvec}_\state}\right)+\E_{\dist}\sum_{y:\bar{\mathbf{y}}_\state\ge\bar{\pvec}_\state}\Pr(Y=\state)\log\left(\frac{f(\pvec)_\state}{\pvec_\state}\right)\\
    =&\sum_{y:\bar{\mathbf{y}}_\state<\bar{\pvec}_\state}\bar{\mathbf{y}}_\state\log\left(\frac{\bar{\mathbf{y}}_\state}{\bar{\pvec}_\state}\right)+\E_{\dist}\sum_{y:\bar{\mathbf{y}}_\state\ge\bar{\pvec}_\state}\Pr(Y=\state)\log\left(\frac{\pvec_y + \epsilon_\pvec \left( \bar{\mathbf{y}}_y - \bar{\pvec}_y \right)}{\pvec_\state}\right)\\
    \ge&\sum_{y:\bar{\mathbf{y}}_\state<\bar{\pvec}_\state}\bar{\mathbf{y}}_\state\log\left(\frac{\bar{\mathbf{y}}_\state}{\bar{\pvec}_\state}\right)+\sum_{y:\bar{\mathbf{y}}_\state\ge\bar{\pvec}_\state}\E_{\dist}[Y_\state]\log\left(\frac{f(\E_{\dist}[\pvec])_\state}{\E_{\dist}[\pvec_\state]}\right)\\
    =&\sum_{y:\bar{\mathbf{y}}_\state<\bar{\pvec}_\state}\bar{\mathbf{y}}_\state\log\left(\frac{\bar{\mathbf{y}}_\state}{\bar{\pvec}_\state}\right)+\sum_{y:\bar{\mathbf{y}}_\state\ge\bar{\pvec}_\state}\bar{\mathbf{y}}_\state\log\left(\frac{\bar{\mathbf{y}}_\state}{\bar{\pvec}_\state}\right)\\
    =& \kl(\bar{\mathbf{y}},\bar{\pvec})\\
\end{align*}

As $\sum_s \pvec_s=1, \sum_s \bar{\pvec}_s=1$ and $\sum_s \bar{\mathbf{y}}_y=1$,

\begin{align*}
    &\pvec_y + \epsilon_\pvec \left( \bar{\mathbf{y}}_y - \bar{\pvec}_y \right)\\
    =&\pvec_y+\frac{(\bar{\mathbf{y}}_y - \bar{\pvec}_y)(1-\pvec_y)}{1-\bar{\pvec}_y}\\
    \coloneqq& A\pvec_y+B
\end{align*}

where $A,B>0$ are constants.

Thus, by the convexity lemma, we can apply Jensen's inequality in $\pvec_y$. This yields the desired lower bound and completes the proof of the optimality.

\end{proof}

\subsection{Performance of Gradient Shift under Relaxed Assumptions}\label{sec:relax}

In this section, we analyze the performance of the gradient shift rule when the assumptions are relaxed. Formally, we will still apply the rule: \begin{align*} \nonumber
\pNew=&f^*(\pStrong,\pWeak,\pPred)\\ 
=&  \proj \left( \gradGinv \big( \gradG(\pStrong) + \gradG(\pWeak) - \gradG(\pPred) \big) \right).
\end{align*} However, $\pPred$ may not be $\E[\pStrongrv\mid \pWeakrv=\pWeak]$ and $\pWeak$ may not be $\E[\statervRawvec\mid \pWeakrv=\pWeak]$.

We will show that in this case, the expected improvement will have an additional residual error term depending on how close $\pPred$, $\pWeak$ are to the true conditional expectations.

\begin{restatable}{theorem}{thmrelaxed}
\label{thm:relaxed}
Let $\score$ be a strictly proper loss induced by $G$ satisfying the regularity conditions. Conditioning on any fixed realization of the helper $\pWeakrv=\pWeak$ and predictor $\pPredrv=\pPred$, let the true moments be $\bar{\pvec} = \E[\pStrongrv\mid \pWeakrv=\pWeak]$ and $\bar{\statevec} = \E[\statervRawvec\mid \pWeakrv=\pWeak]$.  Then, the expected improvement satisfies:
\begin{equation*}
    \E \left[ \score(\pStrongrv, \statervRaw) - \score(\hat{\pNewrv}, \statervRaw) \right] \geq \breg(\bar{\statevec}, \bar{\pvec}) - \breg(\bar{\statevec}, \bar{\statevec}^{out})
\end{equation*}
where $\bar{\statevec}^{out}$ is the shift of the true mean $\bar{\statevec}$, such that $\gradG(\bar{\statevec}^{out}) = \gradG(\bar{\pvec}) + \gap$, where $\gap = \gradG(\pWeak) - \gradG(\pPred)$. The residual $\breg(\bar{\statevec}, \bar{\statevec}^{out})$ vanishes as $(\pPred, \pWeak) \to (\bar{\pvec},\bar{\statevec})$. 
\end{restatable}

\begin{proof}[Proof of \Cref{thm:relaxed}]

We shorthand $\stateRaw$ by $\state$, $\pStrong$ by $\pvec$, $\pPred$ by $\hat{\bar{\pvec}}$, and $\pWeak$ by $\hat{\bar{\statevec}}$. Since $\pWeak$ and $\pPred$ are fixed parameters, we drop them from the input of $f$ for clarity. 

With the above shorthand, $\gap$ becomes $\gradG(\hat{\bar{\statevec}}) - \gradG(\hat{\bar{\pvec}})$.

We analyze the performance of the gradient shift rule under the true moment constraints $\E[\prvvec]=\bar{\pvec}$ and $\E[\statervvec]=\bar{\statevec}$.

Using the three-point identity for Bregman divergences, the expected improvement is:
\begin{align*}
    \E[\breg(\statervvec, \prvvec) - \breg(\statervvec, {f}(\prvvec))] &= \E[\breg({f}(\prvvec), \prvvec) + \langle \gradG({f}(\prvvec)) - \gradG(\prvvec), \statervvec - {f}(\prvvec) \rangle]
\end{align*}

We define the predictor ${f}$ such that:
\begin{align} \label{eq:link_function2}
    \gradG({f}(\pvec)) = \gradG(\pvec) + {\gap}
\end{align}
The expression becomes:
\begin{align} \label{eq:relaxed_expectation_final}
\E\big[\breg(\statervvec,\prvvec)-\breg(\statervvec,f(\prvvec))\big]
&= \E\left[ \breg(f(\prvvec), \prvvec) - \langle {\gap}, f(\prvvec) \rangle \right]
  + \langle {\gap}, \bar{\statevec} \rangle .
\end{align}

Define the bivariate function
\[
\tilde h(\mathbf z,\mathbf q):=\breg(\mathbf z,\mathbf q)-\langle{\gap},\mathbf z\rangle .
\]
Since $\breg(\cdot,\cdot)$ is jointly convex, $\tilde h$ is jointly convex in $(\mathbf z,\mathbf q)$, hence by Jensen,
\begin{align*}
\E\left[\breg({f}(\prvvec),\prvvec)-\langle{\gap},{f}(\prvvec)\rangle\right]
&=\E\left[\tilde h({f}(\prvvec),\prvvec)\right]\\
&\ge \tilde h\left(\E[{f}(\prvvec)],\,\E[\prvvec]\right)\\
&= \breg\left(\E[{f}(\prvvec)],\bar{\pvec}\right)-\langle{\gap},\E[{f}(\prvvec)]\rangle .
\end{align*}
Now consider the (univariate) convex function $h(\mathbf z):=\breg(\mathbf z,\bar{\pvec})-\langle{\gap},\mathbf z\rangle$.

Combining \eqref{eq:relaxed_expectation_final} with the above inequality, we have
\begin{align*}
& \E\big[\breg(\statervvec,\prvvec)-\breg(\statervvec,f(\prvvec))\big]\\
\ge & \Big(\breg\left(\E[{f}(\prvvec)],\bar{\pvec}\right)-\langle{\gap},\E[{f}(\prvvec)]\rangle\Big)+\langle{\gap},\bar{\statevec}\rangle \\
= & h(\E[{f}(\prvvec)])+\langle{\gap},\bar{\statevec}\rangle
\ge \min_{\mathbf z} h(\mathbf z)+\langle{\gap},\bar{\statevec}\rangle .
\end{align*}

To find the minimum of $h(\mathbf{z})$, we set its gradient to zero:
\begin{equation*}
    \nabla h(\mathbf{z}) = \gradG(\mathbf{z}) - \gradG(\bar{\pvec}) - {\gap} = 0 \implies \gradG(\mathbf{z}) = \gradG(\bar{\pvec}) + {\gap}
\end{equation*}
This minimizer is exactly $\bar{\statevec}^{out}$. Substituting this back into the lower bound:
\begin{align} \label{eq:step2_result}
    \E\big[\breg(\statervvec,\prvvec)-\breg(\statervvec,f(\prvvec))\big] \geq \breg(\bar{\statevec}^{out}, \bar{\pvec}) - \langle {\gap}, \bar{\statevec}^{out} \rangle + \langle {\gap}, \bar{\statevec} \rangle = \breg(\bar{\statevec}^{out}, \bar{\pvec}) + \langle {\gap}, \bar{\statevec} - \bar{\statevec}^{out} \rangle
\end{align}

To relate \eqref{eq:step2_result} to the optimal value $\breg(\bar{\statevec}, \bar{\pvec})$, we apply the three-point identity centered at $(\bar{\statevec}, \bar{\pvec}, \bar{\statevec}^{out})$:
\begin{align*}
    \breg(\bar{\statevec}, \bar{\pvec}) &= \breg(\bar{\statevec}, \bar{\statevec}^{out}) + \breg(\bar{\statevec}^{out}, \bar{\pvec}) + \langle \gradG(\bar{\statevec}^{out}) - \gradG(\bar{\pvec}), \bar{\statevec} - \bar{\statevec}^{out} \rangle \\
    &= \breg(\bar{\statevec}, \bar{\statevec}^{out}) + \breg(\bar{\statevec}^{out}, \bar{\pvec}) + \langle {\gap}, \bar{\statevec} - \bar{\statevec}^{out} \rangle
\end{align*}
Rearranging this identity, we find:
\begin{equation*}
    \breg(\bar{\statevec}^{out}, \bar{\pvec}) + \langle {\gap}, \bar{\statevec} - \bar{\statevec}^{out} \rangle = \breg(\bar{\statevec}, \bar{\pvec}) - \breg(\bar{\statevec}, \bar{\statevec}^{out})
\end{equation*}
Substituting this into our bound from Step 2 confirms the improvement:
\begin{equation*}
    \E \left[ \score(\prvvec, \staterv) - \score({f}(\prvvec), \staterv) \right] \geq \breg(\bar{\statevec}, \bar{\pvec}) - \breg(\bar{\statevec}, \bar{\statevec}^{out})
\end{equation*}

By the Generalized Pythagorean Theorem, the Bregman projection ${f}_{proj}(\pvec) = \proj({f}(\pvec))$ onto the simplex $\Delta$ yields a lower (better) loss than the unprojected ones. Thus, the bound holds for the feasible predictor, with the residual $\breg(\bar{\statevec}, \bar{\statevec}^{out})$ vanishing as approximations become exact. ${f}_{proj}(\pvec)$ is precisely the gradient shift rule ${f}^*$. Thus, the bound holds for the gradient shift rule. 
\end{proof}
\section{Deferred Experimental Setup and Results}
\label{sec:expapp}

\subsection{Implementation Details}

\paragraph{Device.} All experiments were run on a single server equipped with one RTX Pro 6000 GPU and an Intel(R) Xeon(R) Platinum 8470Q CPU.

\paragraph{LLM Hyperparameters.} For all non-\textsf{gpt-oss} models, we perform inference with \textsf{bitsandbytes} 8-bit quantization; for \textsf{gpt-oss} models, we use the original \textsf{mxfp4} quantization format.
Unless otherwise stated, we use \texttt{temperature}=0.6 with no \texttt{top\_k} or \texttt{top\_p} truncation, \texttt{frequency\_penalty}=0, and \texttt{presence\_penalty}=0.
We set \texttt{max\_new\_tokens}=1000 for non-\textsf{gpt-oss} models and 2000 for \textsf{gpt-oss} models.

\paragraph{Prompt Template.} For red-teaming jailbreaking, we use the harmful question itself as the prompt (no prompt template).
For math jailbreaking, we use the standard prompt template from the model card of \textsf{ethz-spylab--Llama-3.1-70B-Instruct\_refuse\_math}.

\paragraph{Evaluation Metrics.} For evaluation, HarmBench ASR is computed using \textsf{HarmBench-Llama-2-13b-cls} without additional quantization.
StrongREJECT rubric scores are obtained with the hosted \textsf{gemini-3-flash-preview} model; if the judge model refuses (occurs in $<1\%$ of cases), we fall back to \textsf{devstral-2512}.

\subsection{Additional Evaluation Metrics}

In the red-teaming jailbreaking experiment, we report Attack Success Rate (ASR) using the HarmBench evaluator and harmfulness using the StrongREJECT rubric. To provide a broader view of response harmfulness, we additionally evaluate with two widely used automated judges: \textbf{MD-Judge}~\cite{Li2024SALADBenchAH} and \textbf{HarmScore}~\cite{chan2025speak}. \Cref{fig:redteam_app} complements the main results by showing that the improvements of Hybrid Gradient Shift over Weak-to-Strong persist under these alternative judges, for both a standard LLM and an LRM target.
\begin{figure*}[!ht]
  \centering
  \begin{subfigure}[t]{0.48\textwidth}
    \centering
    \includegraphics[width=\linewidth]{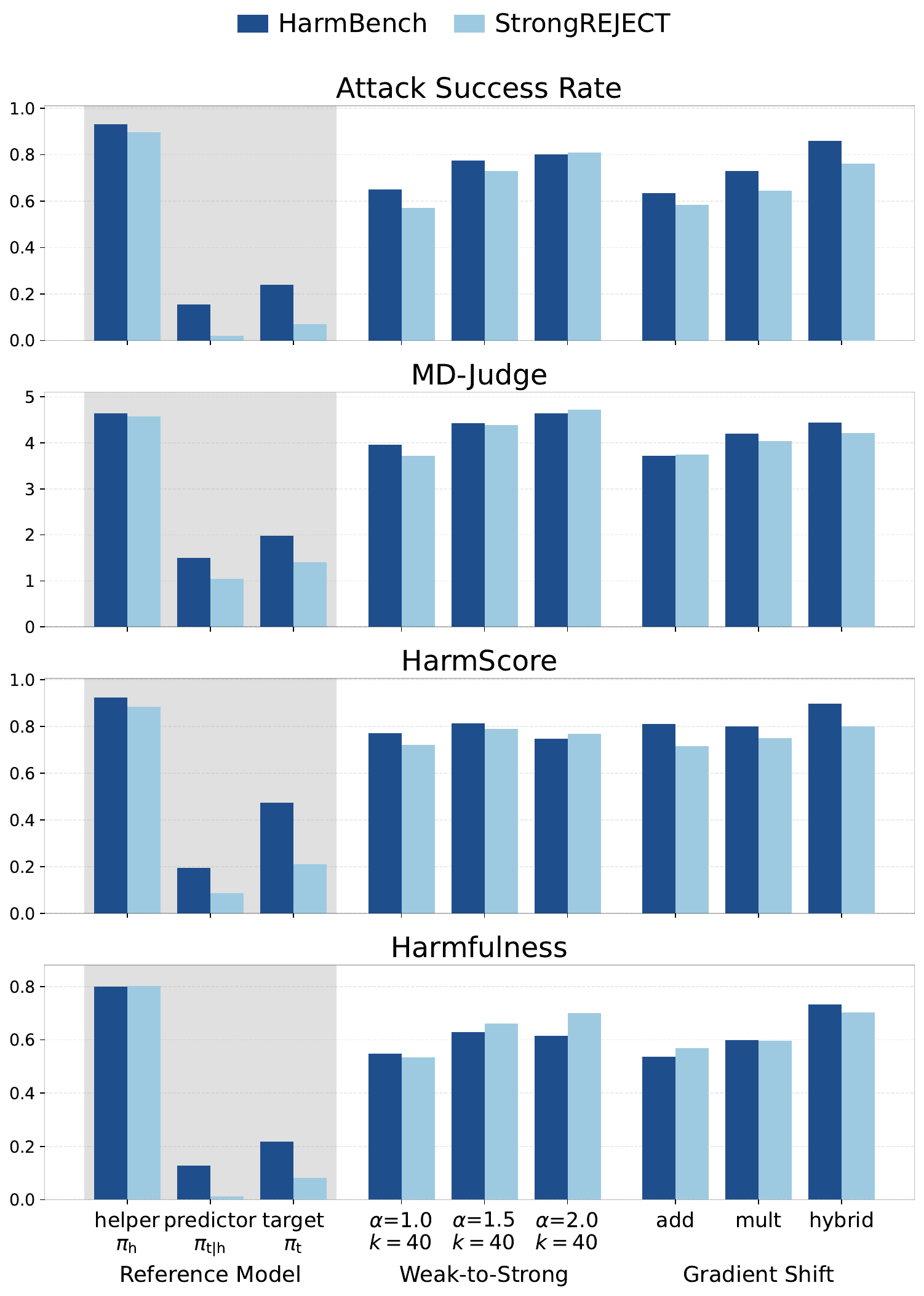}
    \caption{Jailbreaking \textsf{Llama-3.3-70B-Instruct} (standard LLM)}
  \end{subfigure}\hfill
  \begin{subfigure}[t]{0.48\textwidth}
    \centering
    \includegraphics[width=\linewidth]{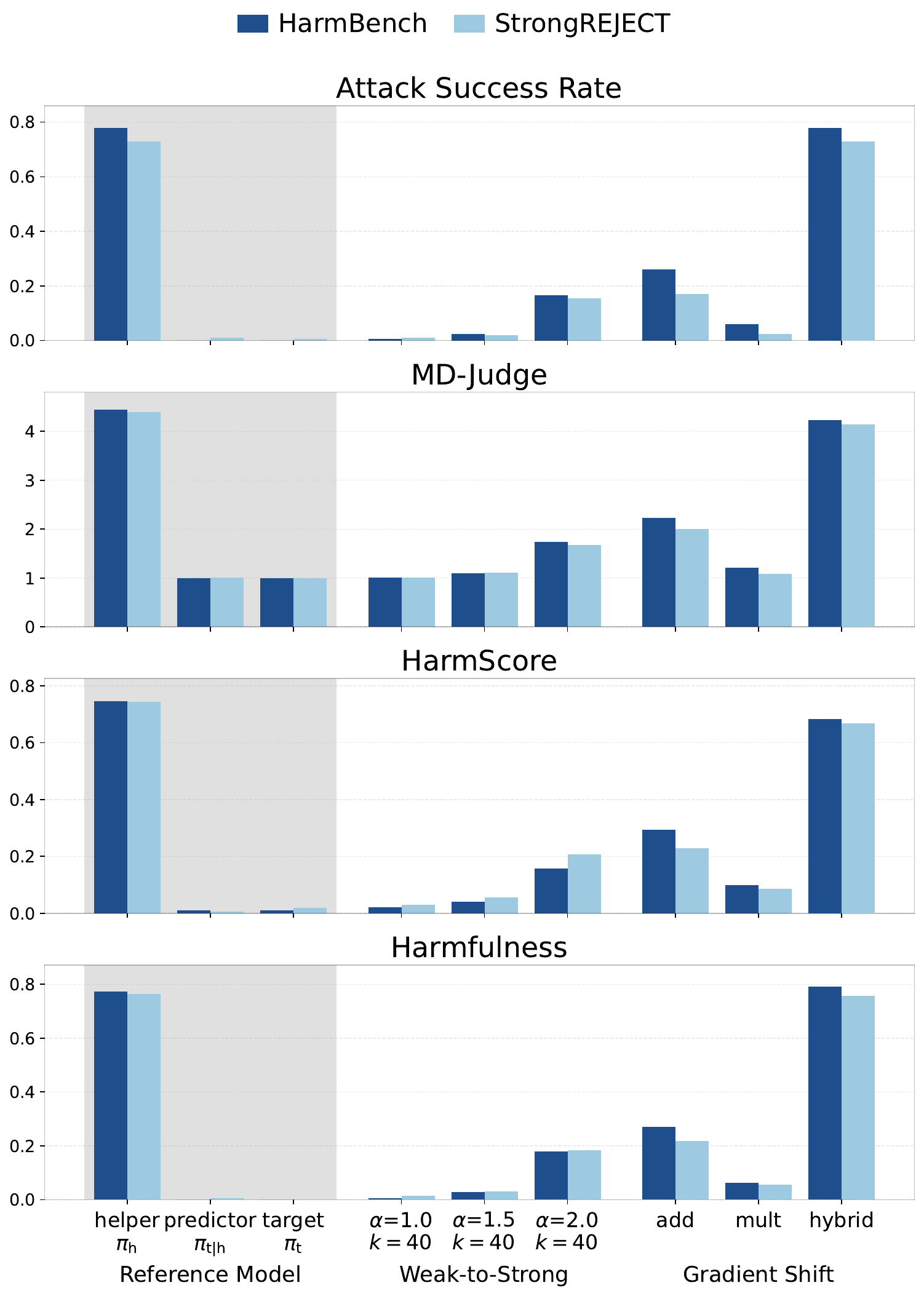}
    \caption{Jailbreaking \textsf{gpt-oss-120b} (LRM)}
  \end{subfigure}
  \caption{\textbf{Red-teaming jailbreaking LLM and LRM}: On both HarmBench and StrongREJECT datasets, our Hybrid Gradient Shift successfully jailbreaks both standard LLM and LRM, outperforming Weak-to-Strong on both ASR and harmfulness.}
  \label{fig:redteam_app}
\end{figure*}

\end{document}